%% file: main_full_version.tex
\documentclass[11pt]{article}
\usepackage{fullpage}
\pdfoutput=1
\usepackage{times}
\usepackage{graphicx,color}
\usepackage{array,float}
\usepackage[hyphens]{url}
\usepackage[usenames,dvipsnames]{xcolor}
\usepackage{amstext,amssymb,amsmath}
\usepackage{amsthm}
\usepackage{mathtools}
\usepackage{verbatim}
\usepackage{bm}
\usepackage{paralist}
\usepackage{ulem}
\usepackage[numbers]{natbib}
\usepackage{hyperref}
\usepackage[noend]{algpseudocode}
\usepackage{algorithm}
\usepackage{xspace}
\usepackage{booktabs}

\newif\ifnotes\notesfalse

\ifnotes
\usepackage[textsize=scriptsize,textwidth=2cm]{todonotes}
\else
\usepackage[disable]{todonotes}
\fi

\newtheorem{lem}{Lemma}
\newtheorem{thm}[lem]{Theorem}
\newtheorem{corollary}[lem]{Corollary}
\newtheorem{defn}[lem]{Definition}
\newtheorem{claim}[lem]{Claim}

\newtheorem{remark}{Remark}

\newcommand{\eps}{\varepsilon}
\newcommand{\Ex}{\mathop{\mathbb{E}}}

\newcommand{\polylog}{\text{polylog}\,}
\newcommand{\state}{\ensuremath{\mathsf{st}}}

\newcommand{\vecmu}{\ensuremath{\vec{\mu}}}
\newcommand{\vecx}{\ensuremath{\vec{x}}}
\newcommand{\vecy}{\ensuremath{\vec{y}}}
\newcommand{\veca}{\ensuremath{\vec{a}}}
\newcommand{\vecb}{\ensuremath{\vec{b}}}
\newcommand{\vecS}{\ensuremath{\vec{S}}}

\newcommand{\A}{\mathcal{A}}

\newcommand{\bE}{\mathbb{E}}

\newcommand{\cA}{\mathcal{A}}

\newcommand{\cD}{\mathcal{D}}

\newcommand{\cS}{\mathcal{S}}

\newcommand{\uniffrom}{\ensuremath{\overset{r}{\leftarrow}}}

\newcommand{\fpriv}{\ensuremath{\tilde{f}}}

\providecommand{\ignore}[1]{\relax}

\newcommand{\edclose}[2]{\approxeq_{(#1, #2)}}

\newcommand{\pnoteinline}[1]{\todo[color=violet!30,inline]{AR: #1}}

\newcommand{\vnote}[1]{\todo[color=blue!40]{VF: #1}}

\newcommand{\remove}[1]{}
\begin{document}

\title{\Large Amplification by Shuffling: \\
From Local to Central Differential Privacy via Anonymity}
\author{
\'Ulfar Erlingsson\thanks{Google Research -- Brain}\\
\and
Vitaly Feldman\footnotemark[1] \\
\and
Ilya Mironov\footnotemark[1] \\
\and
Ananth Raghunathan\footnotemark[1] \\
\and
Kunal Talwar\footnotemark[1] \\
\and
Abhradeep Thakurta\thanks{UC Santa Cruz and Google Research -- Brain.}
}
\date{}
\maketitle

\begin{abstract}
Sensitive statistics are often collected across sets of users, with repeated collection of reports done over time. For example, trends in users' private preferences or software usage may be monitored via such reports. We study the collection of such statistics in the local differential privacy (LDP) model, and describe an algorithm whose privacy cost is polylogarithmic in the number of changes to a user's value.

More fundamentally---by building on anonymity of the users' reports---we also demonstrate how the privacy cost of our LDP algorithm can actually be much lower when viewed in the central model of differential privacy. We show, via a new and general privacy amplification technique, that any permutation-invariant algorithm satisfying $\varepsilon$-local differential privacy will satisfy $(O(\varepsilon \sqrt{\log(1/\delta)/n}), \delta)$-central differential privacy. By this, we explain how the high noise and $\sqrt{n}$ overhead of LDP protocols is a consequence of them being significantly more private in the central model. As a practical corollary, our results imply that several LDP-based industrial deployments may have much lower privacy cost than their advertised $\varepsilon$ would indicate---at least if reports are anonymized.
\end{abstract}

\input{introduction}
\input{background}
\input{longitudinal}
\input{localToCentral-v3}

\input{discussion}

\bibliographystyle{alpha}
\bibliography{longitudinal}

\end{document}

%% file: introduction.tex
\section{Introduction}\label{sec:intro}
A frequent task in data analysis is the monitoring of the statistical properties of evolving data
in a manner that requires repeated computation on the entire evolving dataset.
Software applications commonly apply
online monitoring, e.g.,
to establish trends in the software configurations or usage patterns.
However,
such monitoring may impact the privacy of software users,
as it may
directly or indirectly expose some of their sensitive attributes
(e.g., their location, ethnicity, gender, etc.),
either completely or partially.
To address this,
recent work has proposed
a number of mechanisms
that provide users with strong privacy-protection guarantees
in terms of  of differential privacy~\cite{DMNS,Dwork-ICALP,KLNRS} and,
specifically,
mechanisms that provide
local differential privacy (LDP)
have been deployed by Google, Apple, and Microsoft~\cite{rappor,appledp,DKY17-Microsoft}.

The popularity and practical adoption of
LDP monitoring mechanisms
stems largely from their simple trust model:
for any single LDP report
that a user contributes
about one of their sensitive attributes,
the user will benefit from strong differential privacy guarantees
even if
the user's report becomes public
and all other parties collude against them.

However,
this apparent simplicity
belies the realities of most monitoring applications.
Software monitoring,
in particular,
near always involves repeated collection of reports over time, either
on a regular basis
or triggered by specific software activity;
additionally, not just one, but multiple, software attributes may be monitored,
and these attributes may all be correlated, as well as sensitive,
and may also change in a correlated fashion.
Hence,
a user's actual LDP privacy guarantees
may be dramatically lower than they might appear,
since LDP guarantees
can be exponentially reduced by multiple correlated reports
(see Tang et al.~\cite{korolova-apple} for a case study).
Furthermore, lower accuracy is achieved
by
mechanisms that defend against such privacy erosion
(e.g., the memoized backstop in Google's RAPPOR~\cite{rappor}).
Thus, to square this circle,
and make good privacy/utility tradeoffs,
practical deployments of privacy-preserving monitoring
rely on additional assumptions---in
particular, the assumption
that each user's reports
are anonymous at each timestep
and unlinkable over time.

In this work,
we formalize
how the addition of anonymity guarantees
can improve
differential-privacy protection.
%
Our direct motivation is the
\emph{Encode, Shuffle, Analyze} (ESA) architecture
and \textsc{Prochlo} implementation of Bittau et al.~\citep{prochlo},
which relies on an explicit intermediary
that processes LDP reports from users
to ensure their anonymity.
The ESA architecture
is designed to ensure
a sufficient number of reports are collected at each timestep
so that any one report can ``hide in the crowd''
and to ensure that those reports
are randomly shuffled
to eliminate any signal in their order.
Furthermore,
ESA will also ensure that
reports are disassociated and stripped of
any identifying metadata
(such as originating IP addresses)
to prevent the linking of any two reports to a single user,
whether over time or
within the collection at one timestep.
Intuitively,
the above steps taken to preserve anonymity
will greatly increase
the uncertainty in the analysis of users' reports;
however,
when introducing ESA,
Bittau et al.~\citep{prochlo}
did not show how
that uncertainty could be utilized
to provide
a tighter upper bound on the worst-case privacy loss.

Improving on this,
this paper derives results that
account for the benefits of anonymity
to provide stronger differential privacy bounds.
First,
inspired by
differential privacy under continual observation,
we describe an algorithm
for high-accuracy online monitoring of users' data
in the LDP model
whose total privacy cost is polylogarithmic
in the number of changes to each user's value.
This algorithm shows
how LDP guarantees can be established in online monitoring,
even when users report
repeatedly, over multiple timesteps,
and whether they report on
the same value,
highly-correlated values,
or independently-drawn values.

Second, and more fundamentally,
we show how---when each report is properly
anonymized---any collection of LDP reports
(like those at each timestep of our algorithm above)
with sufficient privacy ($\varepsilon < 1$)
is actually subject to much stronger privacy guarantees
in the central model of differential privacy.
This improved worst-case privacy guarantee
is a direct result of the uncertainty induced by anonymity,
which can prevent reports from any single user
from being singled out or linked together,
whether in the set of reports at each timestep, or over time.


\subsection{Background and related work.}
Differential privacy is a quantifiable measure
of the stability of the output of a
randomized mechanism
in the face of changes to its input data---specifically, when the input from any single user is changed.
(See Section~\ref{sec:backPrelim} for a formal definition.)

\paragraph{Local differential privacy (LDP).}
In the local differential privacy model, formally introduced by Kasiviswanathan et al.~\cite{KLNRS},
the randomized mechanism's output
is the transcript of the entire interaction between a specific user
and a data collector (e.g., a monitoring system).
Even if a user arbitrarily changes their privately held data,
local differential privacy guarantees
will ensure the stability of the distribution of all possible transcripts.
\emph{Randomized response}, a disclosure control technique from the 1960s~\cite{Warner}, is a particularly simple technique for designing single-round LDP mechanisms.
%
Due to their attractive trust model,
LDP mechanisms have recently received
significant industrial adoption for the
privacy-preserving collection of heavy hitters~\cite{rappor,appledp,DKY17-Microsoft},
as well as increased academic attention~\cite{BS15,BST17,Qin:2016,Wang:2017}.


\paragraph{Anonymous data collection.}
As a pragmatic means for reducing privacy risks,
reports are typically anonymized and often aggregated
in deployments of monitoring by careful operators (e.g., RAPPOR~\cite{rappor})---even
though anonymity
is no privacy panacea~\cite{exposed,wsj:location}.

To guarantee anonymity of reports,
multiple mechanisms have been developed,
Many, like Tor~\cite{tor},
are based on the ideas of cryptographic onion routing or mixnets,
often trading off latency
to offer much stronger guarantees~\cite{vuvuzela,stadium,karaoke}.
Some, like those of \textsc{Prochlo}~\cite{prochlo},
are based on oblivious shuffling,
with trusted hardware and attestation used to increase assurance.
Others make use of the techniques of secure multi-party computation,
and can simultaneously aggregate reports and ensure their anonymity~\cite{prio,secureagg}.
Which of these mechanisms is best used in practice
is dictated by what trust model and assumptions apply to any specific deployment.

\paragraph{Central differential privacy.}
The traditional, central model of differential privacy applies to a centrally-held dataset
for which privacy is enforced by a trusted \emph{curator}
that mediates upon queries posed on the dataset by an untrusted \emph{analyst}---with
curators achieving differential privacy
by adding uncertainty (e.g., random noise)
to the answers for analysts' queries.
For differential privacy,
answers to queries
need only be stable with respect to
changes in the data of a single user (or a few users);
these may constitute only a small fraction of the whole, central dataset,
which can greatly facilitate
the establishment of differential privacy guarantees.
%
%
Therefore,
the central model
can offer much better privacy/utility tradeoffs
than the LDP setting.
(In certain cases,
the noise introduced by the curator may even be less than the uncertainty due to population sampling.)

\paragraph{Longitudinal privacy.}

Online monitoring with privacy was formalized by Dwork et al.\
as the problem of differential privacy under continual observation~\cite{Dwork-continual}.
That work proposed a privacy-preserving mechanisms in the central model of differential privacy,
later extended and applied by Chan et al.~\cite{CSS11-continual} and Jain et al.~\cite{JKT-online}.

Continual observations constitute a powerful attack vector. For example, Calandrino et al.~\cite{Calandrino11-like} describe an attack on a collaborative-based recommender system via passive measurements that effectively utilizes differencing between a sequence of updates.

In the local model, Google's RAPPOR~\cite{rappor} proposed a novel \emph{memoization} approach as a backstop against privacy erosion over time: A noisy answer is memorized by the user and repeated in response to the same queries about a data value. To avoid creating a trackable identifier, RAPPOR additionally randomizes those responses, which may only improve privacy. (See Ding et al.~\cite{DKY17-Microsoft} for alternative approach to memoization.) Although memoization prevents a single data value from ever being fully exposed, over time the privacy guarantees will weaken if answers are given about correlated data or sequences of data values that change in a non-independent fashion.

More recently, Tang et al.~\cite{korolova-apple} performed a detailed analysis of
one real-world randomized response mechanisms
and examined its longitudinal privacy implications.

\subsection{Our contributions}
Motivated by the gap in accuracy between central and local differential privacy under continual observations, we describe a general technique for obtaining strong central differential privacy guarantees from (relatively) weak privacy guarantees in the local model.
%
Specifically, our main technical contribution demonstrates that random shuffling of data points ensures
 that the reports from any LDP protocol will also satisfy central differential privacy at a per-report privacy-cost bound that is a factor $\sqrt{n}$ lower than the LDP privacy bound established in the local model.
Here, $n$ is the total number of reports, which can reach into the billions in practical deployments; therefore, the privacy amplification can be truly significant.
%



\paragraph{Privacy amplification by shuffling.}
An immediate corollary of our amplification result is that composing client-side local differential privacy with server-side shuffling allows one to claim strong central differential privacy guarantees \emph{without any explicit server-side noise addition}.

For this corollary to hold,
the LDP reports must be amenable to anonymization via shuffling:
the reports cannot have any discriminating characteristics
and must, in particular, all utilize the same local randomizer
(since the distribution of random values may be revealing).
However,
even if this assumption holds only partially---e.g.,
due to faults, software mistakes, or adversarial control of some reports---the
guarantees  degrade gracefully.
%
Each set of $n'$ users for which the corollary is applicable
(e.g., that utilize the same local randomizer)
will still be guaranteed a factor $\sqrt{n'}$ reduction in their worst-case
privacy cost in the central model.
(See Section \ref{sec:shuffle_after} for a more detailed discussion.)

It is instructive to compare our technique with privacy amplification by subsampling~\cite{KLNRS}.
As in the case of subsampling, we rely on the secrecy of the samples that are used in nested computations. However, unlike subsampling, shuffling, by itself, does not offer any differential privacy guarantees. Yet its combination with a locally differentially private mechanism has an effect that is essentially as strong as that achieved via known applications of subsampling ~\cite{BST14,DP-DL,BalleBG18}.
An important advantage of our reduction over subsampling is that it can include all the data reports (exactly once) and hence need not modify the underlying statistics of the dataset.

In concurrent and independent work, Cheu et al.~\cite{dpmixnets}
formally define an augmented local model of differential privacy that includes an anonymizing shuffler and examine the sample complexity of solving several problems in this model. In this model they demonstrate privacy amplification for a fixed one-bit randomized response mechanism. The analysis in this simple case reduces to a direct estimation of $(\eps,\delta)$-divergence between two binomial distributions. This case was also the starting point of our work. Importantly, in the binary case the bound is $O(\min\{1,\eps_0\} e^{\eps_0/2}\sqrt{\log(1/\delta)/n}$), where $\eps_0$ is the privacy of an individual's randomized response. This bound is asymptotically the same as our general bound of $O((e^{\eps_0}-1)e^{2\eps_0} \sqrt{ \log(1/\delta)/n})$ when $\eps_0\leq 1$ but is stronger than the general bound when $\eps_0 >1$. Finding stronger bounds for general local randomizers in the $\eps_0 >1$ regime is a natural direction for future work.

 \vnote{It seems worth including this analysis in our paper. It has some advantages and we did derive it independently.}

We also remark that another recent privacy amplification technique, via contractive iteration~\cite{FMTT18} relies on additional properties of the algorithm and is not directly comparable to results in this work.
%


\paragraph{Lower bounds in the local model.}
Our amplification result can be viewed, conversely, as giving a lower bound for the local model.
Specifically, our reduction means that lower bounds in the central model translate---with a $\Omega(\sqrt{n})$ \emph{penalty} factor in the privacy parameter---to a large class of local protocols. 
In particular, this suggests that our results in the local model are near-optimal unless the corresponding results in the central model can be improved.
\paragraph{LDP monitoring with longitudinal privacy guarantees.}
%
We introduce an online monitoring protocol
that guarantees longitudinal privacy to users that report over multiple timesteps,
irrespective of whether their reports are about independent or correlated values.
By utilizing our protocol,
users need not worry about revealing too much over time,
and
if they are anonymous,
their reports may additionally
``hide in the crowd'' and
benefit by amplification-by-shuffling,
at least at each timestep.

As a motivating task,
we can consider the
collection of global statistics from users' mobile devices,
e.g., about users' adoption of software apps
or the frequency of users' international long-distance travel.
This task is a natural fit
for
a continual-observations protocol
with LDP guarantees---since
both software use and travel can be highly privacy sensitive---and
can be reduced to collecting a boolean value from each user
(e.g., whether they are in a country far from home).
However, our protocol can be extended to the collection of multi-valued data
or even data strings by building on existing techniques~\cite{rappor,appledp,BST17}.


Concretely, we consider the collection of user statistics across $d$ time periods
(e.g., for $d$ days) with each user \emph{changing} their underlying boolean value
at most $k$ times for some $k \leq d$.
This is the only assumption we place on the data collection task. For software adoption and international travel, the limit on the number of changes is quite natural.
New software is adopted, and then (perhaps) discarded, with only moderate frequency; similarly, speed and distance limit the change rate for even the most ardent travelers.
Formally, we show the following: Under the assumption stated above, one can estimate all the $d$ frequency statistics from $n$ users with error at most $O((\log d)^2k\sqrt{n}/\eps)$ in the local differential privacy model.

Motivated by similar issues, in a recent work Joseph et al.~\cite{JRUW18} consider the problem of tracking a distribution that changes only a small number of times. Specifically, in their setting each user at each time step $t$ receives a random and independent sample from a distribution $P_t$. It is assumed that $P_t$ changes at most $k$ times and they provide utility guarantees that scale logarithmically with the number of time steps. The key difference between this setting and ours is the assumption of independence of each user's inputs across time steps. Under this assumption even a fixed unbiased coin flip would result in each user receiving values that change in most of the steps. Therefore the actual problems addressed in this work are unrelated to ours and we rely on different algorithmic techniques.

\subsection{Organization of the paper}

Section~\ref{sec:backPrelim} introduces our notation and recalls the definition of differential privacy. Section~\ref{sec:logData} provides an algorithm for collecting statistics and proves its accuracy and privacy in the local model under continual observations. Section~\ref{sec:privLDPMix} contains the derivation of our amplification-by-shuffling result.  Section~\ref{sec:discussion} concludes with a discussion. 

%% file: background.tex
\section{Technical Preliminaries and Background}
\label{sec:backPrelim}

\noindent
\textbf{Notation:} For any $n \in \mathbb{N}$, $[n]$ denotes the set $\{1, \ldots, n\}$. For a vector $\vecx$, $x[j]$ denotes the value of its $j$'th coordinate. For an indexed sequence of elements $a_1,a_2,\ldots$ and two indices $i,j$ we denote by $a_{i:j}$ the subsequence $a_i,a_{i+1},\ldots,a_j$ (or empty sequence if $i>j$).
$\|\vecx\|_1$ denotes $\sum |x[j]|$, which is also the Hamming weight of $\vecx$ when $\vecx$ has entries in $\{-1, 0, 1\}$. All logarithms are meant to be base $e$ unless stated otherwise. For a finite set~$X$, let $x \uniffrom X$ denote a sample from $X$ drawn uniformly at random.
\vspace*{1ex}

\noindent
\textbf{Differential privacy:} The notion of differential privacy was introduced by Dwork et al.~\cite{DMNS,Dwork-ICALP}. We are using the (standard) relaxation of the definition that allows for an additive term $\delta$.

\begin{defn}[$(\eps,\delta)$-DP~\cite{ODO}]\label{def:dp} 
A randomized algorithm $\cA\colon \cD^n \to\cS$ satisfies $(\eps,\delta)$-differential privacy (DP) if for all $S\subset\cS$ and for all \emph{adjacent} $D,D'\in \cD$ it holds that
\[
\Pr[\cA(D)\in S]\leq e^\eps\Pr[\cA(D')\in S]+\delta.
\]
\end{defn}
The notion of adjacent inputs is application-dependent, and it is typically taken to mean that $D$ and $D'$ differ in one of the $n$ elements (that corresponds to the contributions of a single individual). We will also say that an algorithm satisfies differential privacy {\em at index $1$} if the guarantees hold only for datasets that differ in the element at index $i$.
We assume the $\eps$ parameter to be a small constant, and $\delta$ is set to be much smaller than $1/|D|$. We repeatedly use the (advanced) composition property of differential privacy.

\begin{thm}[Advanced composition~\cite{dwork2010boosting, dwork2014algorithmic}]\label{thm:composition}
If $\cA_1, \ldots, \cA_k$ are randomized algorithms satisfying $(\eps, \delta)$-DP, then their composition, defined as $\left(\cA_1(D), \ldots, \cA_k(D)\right)$ for $D \in \cD$ satisfies $(\eps', k\delta + \delta')$ differential privacy where $\eps' = \eps \sqrt{2k\log(1/\delta')} + k\eps(\exp(\eps)-1)$. Moreover, $\cA_i$ can be chosen \emph{adaptively} depending on the outputs of $\cA_1, \ldots, \cA_{i-1}$. 
\end{thm}

A straightforward corollary implies that for $\eps, \eps' < 1$, a $k$-fold composition of $(\eps, \delta)$-DP algorithms leads to an $O(\sqrt{k\log(1/\delta)})$ overhead, i.e., $\eps' = 2\eps\sqrt{2k\log(1/\delta)}$. 

It will also be convenient to work with the notion of distance between distributions on which $(\eps,\delta)$-DP is based more directly. We define it below and describe some of the properties we will use. Given two distributions $\mu$ and $\mu'$, we will say that they are $(\eps, \delta)$-DP close, denoted by $\mu \edclose{\eps}{\delta} \mu'$, if for all measurable $A$, we have
\begin{align*}
    \exp(-\eps)(\mu'(A) - \delta) \leq \mu(A) \leq \exp(\eps) \mu'(A) + \delta.
\end{align*}
For random variables $X, X'$, we write $X \edclose{\eps}{\delta} X'$ to mean that their corresponding distributions $\mu, \mu'$ are $(\eps, \delta)$-DP close. We use $X \approxeq X'$ to mean that the random variables are identically distributed. 

For distributions $\mu_1, \mu_2$ and $a\in [0,1]$, we write $a \mu_1 + (1-a) \mu_2$ to denote the mixture distribution that samples from $\mu_1$ with probability $a$ and from $\mu_2$ with probability $(1-a)$.
The following properties are well-known properties of $(\eps, \delta)$-DP.
\begin{lem} The notion of $(\eps, \delta)$-DP satisfies the following properties:
\begin{description}
\item[Monotonicity] Let $\mu \edclose{\eps}{\delta} \mu'$. Then for $\eps' \geq \eps$ and $\delta' \geq \delta$, $\mu \edclose{\eps'}{\delta'} \mu'$.
\item[Triangle inequality] Let $\mu_1 \edclose{\eps_1}{\delta_1} \mu_2$ and $\mu_2 \edclose{\eps_2}{\delta_2} \mu_3$. Then  $\mu_1 \edclose{\eps_1+\eps_2}{\delta_1+\delta_2} \mu_3$.
\item[Quasi-convexity] Let $\mu_1 \edclose{\eps}{\delta} \mu'_1$ and $\mu_2 \edclose{\eps}{\delta} \mu'_2$, then for any $a \in [0,1]$, it holds that $(1-a) \mu_1 + a \mu_2 \edclose{\eps}{\delta} (1-a) \mu'_1 + a \mu'_2$.
\remove{    
\item[Post-processing] Let $X \edclose{\eps}{\delta} X'$. Then for any Markov kernel (or, equivalently, randomized algorithm) $\A$, it holds that $\A(X) \edclose{\eps}{\delta} \A(X')$.
\item[Expectation closeness] Let $X,X'$ take values in $[0,1]$ and $X \edclose{\eps}{\delta} X'$. Then $\Ex[X] \leq \exp(\eps)\Ex[X'] + \delta$.}
\end{description}
\end{lem}

The following lemma is a reformulation of the standard privacy amplification-by-sampling result \cite{KLNRS} (with the tighter analysis from \cite{Ullman-lecturenotes}).
\begin{lem}[\cite{KLNRS,Ullman-lecturenotes}]
\label{lem:pas}
Let $q < \frac12$ and let $\mu_0, \mu_1$ be distributions such that $\mu_1 \edclose{\eps}{\delta} \mu_0$. For $\mu = (1-q) \mu_0 + q \mu_1$, it holds that $\mu \edclose{\eps'}{q\delta} \mu_0$, where $\eps' = \log(q(e^\eps-1)+1) \leq q(e^\eps-1)$.
\end{lem}

%% file: longitudinal.tex
\section{Locally Private Protocol for Longitudinal Data}
\label{sec:logData}

Recall the motivation for collecting statistics from user devices with the intent of tracking global trends. We remain in the local differential privacy model, 
but our protocol addresses the task of collecting reports from users 
to derive global statistics that are \emph{expected} 
to change across time.
We consider the simplified task of collecting a boolean value from each user, e.g., their device being at an international location, far from the user's home.  However, our protocol can be straightforwardly extended to collecting richer data, such as strings, 
by building on existing techniques~\cite{rappor,appledp,BST17}. 

In what follows, we consider a natural model of collecting user statistics across time periods. We make two minimal assumptions: we are given the time horizon $d$, or the number of time periods (or days) ahead of time, and each user changes their underlying data at most $k \leq d$ times. The first assumption is mild: a loose upper bound on $d$ suffices, and the error depends only polylogarithmically on the upper bound. The second assumption can be enforced at the client side to ensure privacy, while suffering some loss in accuracy.

Our approach is inspired by the work on privacy under continual observations by Dwork et al.~\cite{Dwork-continual}, who give a (central) DP mechanism to maintain a counter that is incrementally updated in response to certain user-driven events. Its (central) differential privacy is defined in respect to a single increment, the so-called event-level privacy. The na\"ive solution of applying additive noise to all partial sums introduces error $\Theta(\sqrt{d})$,  proportional to the square root of the time horizon.
The key algorithmic contribution of Dwork et al.\ is an elegant aggregation scheme that reduces the problem of releasing partial sums to the problem of maintaining a binary tree of counters. By carefully correlating noise across updates, they reduce the error to $O(\polylog{d})$. (In related work Chan et al.~\cite{CSS11-continual} describe a post-processing procedure that guarantees output consistency; Xiao et al.~\cite{XWG-wavelets} present a conceptually similar algorithm framed as a basis transformation.)

We adapt these techniques to the local setting by pushing the tree-of-counters to the client (Algorithm~\ref{alg:marginalCountClientSideOnline} below). The server aggregates clients' reports and computes the totals (Algorithm~\ref{alg:marginalCountServerSideOnline}).

\paragraph{Setup.} We more formally define the problem of collecting global statistics based on reports from users' devices. Given a time horizon~$d$, we consider a population of $n$ users reporting a boolean value about their state at each time period $t \in [d]$. (Without loss of generality, we assume that $d$ is a power of 2.) Let $\vec{\state}_i = \left[\state_i[1], \ldots, \state_i[d] \right]$ denote the states of the $i$'th user across the $d$ time periods with at most $k$ changes. The task of collecting statistics requires the server to compute the sum $\sum_{i \in [n]} \state_i[t]$ for every time periods $t$.

For the reason that will become clear shortly, it is convenient to consider the setup where users report only changes to their state, i.e., a finite derivative of $\vec{\state}_i$. Let $\vecx_i = \left[x_i[1], \ldots, x_i[d]\right] \in \{-1, 0, 1\}^d$ denote the changes in the $i$'th user's state between consecutive time periods. Our assumption implies that each $\vecx_i$ has at most $k$ non-zero entries. It holds that $\state_i[t] = \sum_{\ell \in [t]} x_i[\ell]$ for all $t \in [d]$. Let $f_t = \sum_{i=1}^n x_i[t]$. For the collection task at hand, it suffices to estimate ``running counts'' or marginal sums $\{f_t\}_{t \in [d]}$. 

An online client-side algorithm for reporting statistics runs on each client device and produces an output for each time period. Correspondingly, the online server-side algorithm receives reports from $n$ clients and outputs estimates for the marginal $f_t$ at each time period $t$.\\

\paragraph{Outline.} To demonstrate the key techniques in the design of our algorithm, consider a version of the data collection task with every client's data known ahead of time. Given $\vec{x}_i$ for user $i$, the client-side algorithm produces (up to) $d$ reports and the server computes estimates of the marginal sum $f_t$ for all $t \in [d]$. Our algorithm is based on the tree-based aggregation scheme~\cite{Dwork-continual,CSS11-continual} used previously for releasing continual statistics over longitudinal data in the central model. Each client maintains a binary tree over the $d$ time steps to track the (up to) $k$ changes of their state. The binary tree ensures that each change affects only $\log_2 d$ nodes of the tree. We extend the construction of Dwork et al.~\cite{Dwork-continual} in a natural manner to the local model by having each client maintain and report values in this binary tree with sub-sampling as follows.

In the beginning, the client samples uniformly
from $[k]$ the $\kappa^*$'th change they would like to report on. Changes other than the $\kappa^*$'th one are ignored. The client builds a tree with leaves corresponding to an index vector capturing the $\kappa^*$'th change ($0$ everywhere except $\pm 1$ at the change). The rest of the nodes are populated with the sums of their respective subtrees. The client then chooses a random level of the tree to report on. 
Then, the client runs randomized response on each node of the selected level (with noise determined by $\eps$) and reports the level of the tree along with the randomized response value for each node. In actual implementations and our presentation of the protocol (Algorithm~\ref{alg:marginalCountClientSideOnline}) the tree is never explicitly constructed. The state maintained by the client consists of just four integer values ($\kappa^*$, the level of the tree, and two counters).

The server accumulates reports from clients to compute an aggregate tree comprising sums of reports from all clients. 
To compute the marginal estimate~$\tilde{f}_t$ for the time step~$t$, the server sums up the respective internal nodes whose subtrees form a disjoint cover of the interval $[1, t]$ and scales it up by the appropriate factor (to compensate for the client-side sampling). 



\paragraph{Notation.} To simplify the description of what follows, for a given $d$ (that is a power of two), we let $h \in [\log_2(d) + 1]$ (and variants such as $h_i$ and $h^*$) denote the $h$'th level of a balanced binary tree with $2d$ nodes where leaves have level 1. We let $H(h)$ denote the value $d/2^{h-1}$, the number of nodes at level $h$, and write~$H_i$ (resp.~$H^*$) to denote $H(h_i)$ (resp.~$H(h^*)$). We let $[h,j]$, for $h \in [\log_2(d)+1]$ and $j \in [H(h)]$ denote the $j$'th node at level $h$ of the binary tree $T$, and $T[h,j]$ as the corresponding value stored at the node. Algorithm~\ref{alg:marginalCountClientSideOnline} describes the client-side algorithm to generate differentially-private reports and Algorithm~\ref{alg:marginalCountServerSideOnline} describes the server-side algorithm that collects these reports and estimates marginals $\tilde{f}_t$ for each time period. Theorems~\ref{thm:ldpLongPriv} and~\ref{thm:ldpLongUtil} state the privacy and utility guarantees of the algorithms.

\begin{algorithm}[!htbp]
    \caption{($\cA_{\sf client}$) : Locally differentially private reports.}
    \label{alg:marginalCountClientSideOnline}
    \begin{algorithmic}[1]
    \Procedure{Setup}{$d, k$}
        \Require Time bound $d$; bound on the number of non-zero entries in $\vec{x}$: $k\geq \|\vec{x}\|_0$.
        
        \State\label{alg:pick_i_h} Sample $\kappa^* \uniffrom [k]$ and $h^* \uniffrom [\log_2(d) + 1]$
        \State Initialize counters $\kappa \gets 0$ and $c\gets 0$
    \EndProcedure
    \Statex
    \Procedure{Update}{$t, x_t,\eps$}
    \Require{Time $t\leq d$, $x_t\in\{-1,0,1\}$, privacy budget $\eps$.}
    \Ensure{Modifies counters $\kappa$ and $c$.}
    \If{$x_t\neq 0$}
        \State $\kappa\gets \kappa + 1$\Comment{$\kappa$ tracks the number of non-zeroes}
        \If{$\kappa = \kappa^*$}
            \State $c\gets x_t$
        \EndIf
    \EndIf
    \If{$t$ is divisible by $2^{h^*-1}$}
        \If{$c=0$}\label{step:c=0}
            \State $u\uniffrom \{-1, 1\}$\label{step:pm1}
        \Else
            \State\label{step:rr1-online} $b\gets 2\cdot\mathrm{B}\left(\frac{e^{\eps/2}}{1+e^{\eps/2}}\right)-1$\Comment{$\mathrm{B}(p)$---Bernoulli r.v.\ with expectation $p$}\label{step:bernoulli}
            \State\label{step:rr2-online} $u\gets b\cdot c$
            \State $c\gets 0$\Comment{$c$ will never be non-zero again}
        \EndIf
        \State\label{step:report-online} \textbf{report} $(h^*, t, u)$
    \EndIf
    \EndProcedure
    \end{algorithmic}
\end{algorithm}

\begin{algorithm}[!htbp]
    \caption{($\cA_{\sf server}$) :  Aggregation of locally differentially private reports (server side)}
    \label{alg:marginalCountServerSideOnline}
    \begin{algorithmic}[1]
        \Require For every $t \in [d]$, reports $(h_{i, t}, t, u_{i, t})$ from the $i$'th client running $\cA_\mathsf{client}$. (Some of these reports can be empty.) Privacy budget $\eps$, bound $k$.
        \State Create a balanced binary tree $T_\mathsf{sum}$ with $d$ leaves. 
        \For{$t$ \textbf{in} $[d]$}
        \For{$h$ s.t.\ $2^{h-1}$ divides $t$}
        \State $T_\mathsf{sum}[h, t / 2^{h-1}] \gets \sum_{i\colon h_{i,t}=h} u_{i, t}$\Comment{Accumulate all reports from level $h$ and time $t$}
        \EndFor
        \EndFor

        \For{$t$ \textbf{in} $[d]$}
     \State\label{step:C-start} Initialize $C\gets \{[1,1],[1,2],\dots,[1,t]\}$.
            \While{$h$ and even $i$ exist s.t.\ $[h,i-1],[h,i]\in C$}
            \State Remove $[h,i-1],[h,i]$ from $C$. 
            \State Add $[h+1,i / 2]$ to $C$. 
            \EndWhile\label{step:C-stop}
            \State $\fpriv_t\gets \frac{e^{\eps/2}+1}{e^{\eps/2}-1} k(\log_2d)\cdot\sum_{[h,i]\in C} T_{\sf sum}[h,i]$
            \State \textbf{report} privately estimated marginal $\fpriv_t$
        \EndFor
    \end{algorithmic}
\end{algorithm}

\begin{thm}[Privacy]
    The sequence of $d$  outputs of $\cA_{\sf client}$ (Algorithm~\ref{alg:marginalCountClientSideOnline}) satisfies $\eps$-local differential privacy. \label{thm:ldpLongPriv}
\end{thm}

\begin{proof}
To show the local differential privacy of the outputs, we consider two aspects of the client's reports: (1) the (randomized) values at the nodes output in Step~\ref{step:report-online}, and (2) the timing of the report. The latter entirely depends on the choice of $h^*$ which is independent of the client's underlying data record and hence does not affect the privacy analysis. Furthermore, $\kappa^*$ is also independently sampled and for the rest of the proof, we fix $\kappa^*$ and $h^*$ and focus only on the randomized values output in Step~\ref{step:report-online}.

By construction, the client chooses only the $\kappa^*$'th change to report on. This would imply that two inputs would affect at most two nodes at level $h^*$ with each of the values changing by at most one. 

This bound on the sensitivity of the report enables us to use standard arguments for randomized response mechanisms~\cite{rappor} to Steps~\ref{step:rr1-online} and~\ref{step:rr2-online} to show that the noisy values $T[h^*,1], \ldots, T[h^*,H^*]$ satisfy $\eps$-local differential privacy.


\end{proof}


\begin{thm}[Utility] For $\eps \leq 1$, with probability at least $2/3$ over the randomness of $\cA_{\sf client}$ run on $n$ data records, the outputs $\fpriv_1,\dots,\fpriv_d$ of $\cA_{\sf server}$ satisfy:
\begin{equation*} 
\max_{t\in[d]}\left|f_t-\fpriv_t\right|=O\left(c_\eps(\log d)^2 k\sqrt{n}\right),
\end{equation*} where $c_\eps=\frac{e^{\eps/2}+1}{e^{\eps/2}-1}$.
    \label{thm:ldpLongUtil}
\end{thm}

\begin{proof}

The leaves of the binary tree with nodes $T[h,t]$ in Algorithms~\ref{alg:marginalCountClientSideOnline} and~\ref{alg:marginalCountServerSideOnline} comprise events in each of the time periods $[1, d]$. The marginal counts $f_t$ and $\fpriv_t$ comprise the exact (resp., approximate) number of events observed by all clients in the time period $[1,t]$.

Consider the set $C$ constructed in Steps~\ref{step:C-start}--\ref{step:C-stop} of Algorithm~\ref{alg:marginalCountServerSideOnline}. We observe that $C$ satisfies the following properties:
\begin{itemize}
    \item The set $C$ is uniquely defined, i.e., it is independent of the order in which pairs of intervals are selected in the \textbf{while} loop.
    \item The size of the set $|C|$ is at most $\log_2d$.
    \item The leaf nodes of the subtrees $[h,i]$ in $C$ partition $[1,t]$.
    \end{itemize}

The marginal counts $f_t$ totaling events in the interval~$[1, t]$ can be computed by adding the corresponding~$\log_2 d$ nodes in the tree whose subtrees partition the interval~$[1,t]$.

It follows that largest error in any $f_t$ is at most $\log_2 d$ times the largest error in any of the subtree counts. We proceed to bound that error.

For a node $[h,j]$ in the tree, let $S[h,j]$ denote the sum $\sum_{i}\sum_{t \in [h,j]} x_i[t]$, i.e., the actual sum of $x_i[t]$ values in the subtree at $[h, j]$, summed over all $i$. Let $z_i^{[h,j]} = \sum_{t \in [h,j]} x_i[t]$ denote the contribution of client $i$ to $S[h,j]$. We will argue that $|S[h,j] - c_\eps k\log_2(d)\cdot T[h,j]|$ is small with high probability, which would then imply the bound on $|f_t - \fpriv_t|$.

Towards that goal, for a client $i$, and node $[h, j]$, let $u_{i}^{[h,j]}$ be the contribution of client $i$ to the sum $T[h,j]$ computed by the server. Thus for any $h \neq h^*_i$, this value is zero, and for $h = h^*_i$, $u_i^{[h,j]}$ is a randomized response as defined in steps \ref{step:c=0}--\ref{step:rr2-online} of Algorithm~\ref{alg:marginalCountClientSideOnline}. Clearly $u_i^{[h,j]} \in \{-1, 0, 1\}$. We next argue that $c_\eps k\log_2(d)\Ex[u_i^{[h,j]}]$ is  equal to~$z_i^{[h, j]}$. Indeed, for a specific $x_i[t]$ to have an effect on $T[h,j]$, we should have this be the $\kappa^*$'th non-zero entry in $x_i$ (which happens with probability~$\frac 1 k$). Further, $h$ should equal $h^*_i$ (which happens with probability~$\frac 1 {\log_2 d}$). Conditioned on these two events, the expected value of $u_i^{[h,j]}$ is determined in step~\ref{step:bernoulli} as $x_i[t]$ multiplied by $2\frac{e^{\eps/2}}{e^{\eps/2}+1} - 1 =\frac1{c_\eps}$. When one of these two events fails to happen, the expected value of $u_i^{[h,j]}$ is determined in line~\ref{step:pm1}, and is zero. It follows that the expectation of $u_i^{[h,j]}$ is as claimed and thus the expectation of $T[h,j]$ is exactly $\frac{S[h,j]}{c_\eps k\log_2(d)}$.

To complete the proof, we note that $T[h,j]$ is the sum of independent (for each $i$) random variables that come from the range $[-1,1]$. Standard Chernoff bounds then imply that $T[h,j]$ differs from its expectation by at most $c\sqrt{3\Ex[T[h,j]]\log \frac 1 {\beta'}}$, except with probability $\beta'$. This expectation is bounded by $O_{\eps}(n /\log_2 d)$. Setting $\beta' = \beta/2d$, and taking a union bound, we conclude that except with probability $\beta$:
\begin{align*}
    \forall [h,j]\colon |T[h, j] - \Ex[T[h,j]]| \leq c_{\eps} \sqrt{n\log \frac {2d} \beta / \log_2 d}.
\end{align*}
Scaling by $c'_{\eps} k\log_2 d$, and multiplying by $\log_2 d$ to account for the size of $C$, we conclude that except with probability $\beta$,
\begin{align*}
    \forall t \in [d]\colon |f_t - \fpriv_t| \leq c_{\eps} k (\log_2 d)^{\frac 3 2} \sqrt{n \log \frac {2d}{\beta}}.
\end{align*}
The claim follows.
\end{proof}

%% file: localToCentral-v3.tex
\section{Privacy Amplification via Shuffling}
\label{sec:privLDPMix}

\newcommand{\Apost}{{\cA_{\mathsf{post}}}}
\newcommand{\Asl}{{\cA_{\mathsf{sl}}}}
\newcommand{\Aloc}{{\cA_{\mathsf{local}}}}
\newcommand{\Asw}{{\cA_{\mathsf{swap}}}}
\newcommand{\Aldp}{\cA_{\mathsf{ldp}}}
\newcommand{\Balg}{\mathcal{B}}

Local differential privacy holds against a strictly more powerful adversary than central differential privacy, namely one that can examine all communications with a particular user. What can we say about central DP guarantees of a protocol that satisfies $\eps$-local DP? Since local DP implies central DP, this procedure satisfies $\eps$-central DP. Moreover, without making additional assumptions, we cannot make any stronger statement than that. Indeed, the central mechanism may release the entire transcripts of its communications with all its users annotated with their identities, leaving their local differential privacy as the only check on information leakage via the mechanism's outputs.

We show that, at a modest cost and with little changes to its data collection pipeline, the analyst can achieve much better central DP guarantees than those suggested by this conservative analysis. Specifically, by shuffling (applying a random permutation), the analyst may \emph{amplify} local differential privacy by a $\tilde{\Theta}(\sqrt{n})$ factor.

To make the discussion more formal we define the following general form of locally differentially private algorithms that allows picking local randomizers sequentially on the basis of answers from the previous randomizers (Algorithm \ref{alg:local}).
\begin{algorithm}[!htbp]
    \caption{\textbf{:} $\Aloc$: Local Responses}
    \begin{algorithmic}[1]
        \Require Dataset $D = x_{1:n}$. Algorithms $\Aldp^{(i)}\colon \cS^{(1)}\times \cdots \times \cS^{(i-1)} \times  \cD\to\cS^{(i)}$ for $i\in [n]$.
        \For{$i$ \textbf{in} $[n]$}
        \State $z_i \leftarrow \Aldp^{(i)}(z_{1:i-1}; x_i)$
       \EndFor
       \State \Return sequence $z_{1:n}$
    \end{algorithmic}
    \label{alg:local}
\end{algorithm}

Our main amplification result is for Algorithm \ref{alg:shuffle} that applies the local randomizers {\em after} randomly permuting the elements of the dataset. For algorithms that use a single fixed local randomizer it does not matter whether the data elements are shuffled before or after the application of the randomizer.
We will discuss the corollaries of this result for algorithms that shuffle the responses later in this section.
\begin{algorithm}[!htbp]
    \caption{\textbf{:} $\Asl$: Local Responses with Shuffling}
    \begin{algorithmic}[1]
        \Require Dataset $D = x_{1:n}$. Algorithms $\Aldp^{(i)}\colon \cS^{(1)}\times \cdots \times \cS^{(i-1)} \times \cD\to\cS^{(i)}$ for $i\in [n]$.
        \State Let $\pi$ be a uniformly random permutation of $[n]$.
        \State  $\pi(D) \gets (x_{\pi(1)},x_{\pi(2)},\ldots,x_{\pi(n)})$
        \State \Return $\Aloc(\pi(D))$
    \end{algorithmic}
    \label{alg:shuffle}
\end{algorithm}
\begin{thm}[Amplification by shuffling]
For a domain $\cD$, let $\Aldp^{(i)}\colon \cS^{(1)}\times \cdots \times \cS^{(i-1)} \times  \cD\to\cS^{(i)}$ for $i \in [n]$ (where $\cS^{(i)}$ is the range space of $\Aldp^{(i)}$) be a sequence of algorithms such that $\Aldp^{(i)}$ is $\eps_0$-differentially private for all values of auxiliary inputs in $\cS^{(1)}\times \cdots \times \cS^{(i-1)}$. Let $\Asl \colon\cD^n\to \cS^{(1)}\times \cdots \times \cS^{(n)}$ be the algorithm that given a dataset $x_{1:n} \in \cD^n$,  samples a uniform random permutation $\pi$ over $[n]$, then sequentially computes $z_i = \Aldp(z_{1:i-1}, x_{\pi(i)})$ for $i=1,2,\ldots,n$ and outputs $z_{1:n}$ (see Algorithm~\ref{alg:shuffle}). For any integer $n>1$, $\eps_0,\delta>0$, $\Asl$ satisfies $(\eps,\delta)$-differential privacy in the central model, where $\eps \leq  \eps_1\sqrt{2 n \log(1/\delta)} + n \eps_1 (e^{\eps_1}-1)$ for $\eps_1 = 2e^{2\eps_0}(e^{\eps_0}-1)/n$. In particular,
for $\eps_0 \leq \ln(n/4)/3$, $\eps \leq e^{2\eps_0}(e^{\eps_0}-1)\sqrt{\frac{8\log(1/\delta)}{n}} + 6 e^{4\eps_0}(e^{\eps_0}-1)^2/n$ and for any $n\geq 1000$, $0<\eps_0<1/2$ and $0<\delta<1/100$, $\eps \leq 12 \eps_0 \sqrt{\frac{\log(1/\delta)}{n}}$.
\label{thm:local-to-central}
\end{thm}
We remark that while we state our amplification result for local randomizers operating on a single data element, the result extends immediately to arbitrary $\eps_0$-DP algorithms that operate on disjoint batches of data elements. Also note that for $\eps_0 > \ln(n/2)/3$,  the resulting bound on $\eps$ is larger than $\eps_0$ and thus our bound does not provide any amplification.

A natural approach to proving this result is to use privacy amplification via subsampling (Lemma \ref{lem:pas}). At step~$i$ in the algorithm,  conditioned on the values of $\pi(1), \ldots, \pi(i-1)$, the random variable $\pi(i)$ is uniform over the remaining $(n-i+1)$ indices. Thus the $i$'th step will be $O((e^{\eps_0}-1)/(n-i+1))$-DP. The obvious issue with this argument is that it gives very little amplification for values of $i$ that are close to $n$, falling short of our main goal. It also unclear how to formalize the intuition behind this argument.

Instead our approach crucially relies on a reduction to analysis of the algorithm $\Asw$ that swaps the first element in the dataset with a uniformly sampled element in the dataset before applying the local randomizers (Algorithm \ref{alg:swap}). We show that $\Asw$ has the desired privacy parameters for the first element (that is, satisfies the guarantees of differential privacy only for pairs of datasets that differ in the first element). We then show that for every index~$i^\ast$, $\Asl$ can be decomposed into a random permutation that maps element~$i^\ast$ to be the first element followed by~$\Asw$. This implies that the algorithm~$\Asl$ will satisfy differential privacy at~$i^\ast$.

To argue about the privacy properties of the $\Asw$ we decompose it into a sequence of algorithms, each producing one output symbol (given the dataset and all the previous outputs). It is not hard to see that the output distribution of the $i$'th algorithm is a mixture $\mu = (1-p)\mu_0 + p\mu_1$, where $\mu_0$ does not depend on $x_1$ (the first element of the dataset) and $\mu_1$ is the output distribution of the $i$'th local randomizer applied to~$x_1$. We then demonstrate that the probability $p$ is upper bounded by $e^{2\eps_0}/n$.
Hence, using amplification by subsampling, we obtain that $\mu$ is close to $\mu_0$. The distribution $\mu_0$ does not depend on $x_1$ and therefore, by the same argument, $\mu_0$ (and hence $\mu$) is close to the output distribution of the $i$'th algorithm on any dataset $D'$ that differs from $D$ only in the first element. Applying advanced composition to the sequence of algorithms gives the claim.

We now provide the full details of the proof starting with the description and analysis of $\Asw$.
\begin{algorithm}[!htbp]
    \caption{\textbf{:} $\Asw$: Local responses with one swap}
    \begin{algorithmic}[1]
        \Require Dataset $D = x_{1:n}$. Algorithms $\Aldp^{(i)}\colon \cS^{(1)}\times \cdots \times \cS^{(i-1)} \times \cD\to\cS^{(i)}$ for $i\in [n]$.
        \State Sample $I\uniffrom [n]$
        \State Let $\sigma_I(D)\gets (x_{I},x_2,\dots,x_{I-1},x_1,x_{I+1},\dots,x_n)$
        \State\Return $\Aloc(\sigma_I(D))$
    \end{algorithmic}
    \label{alg:swap}
\end{algorithm}
\begin{thm}[Amplification by swapping]
For a domain $\cD$, let $\Aldp^{(i)}\colon \cS^{(1)}\times \cdots \times \cS^{(i-1)} \times  \cD\to\cS^{(i)}$ for $i \in [n]$ (where $\cS^{(i)}$ is the range space of $\Aldp^{(i)}$) be a sequence of algorithms such that $\Aldp^{(i)}$ is $\eps_0$-differentially private for all values of auxiliary inputs in $\cS^{(1)}\times \cdots \times \cS^{(i-1)}$. Let $\Asw \colon\cD^n\to \cS^{(1)}\times \cdots \times \cS^{(n)}$ be the algorithm that given a dataset $D=x_{1:n} \in \cD^n$,  samples a uniform index $I \in [n]$, swaps element $1$ with element $I$ and then applies the local randomizers to the resulting dataset sequentially (see Algorithm~\ref{alg:swap}).
For any integer $n>1$, $\eps_0,\delta>0$, $\Asw$ satisfies $(\eps,\delta)$-differential privacy at index $1$ in the central model, where $\eps \leq  \eps_1\sqrt{2 n \log(1/\delta)} + n \eps_1 (e^{\eps_1}-1)$ for $\eps_1 = 2e^{2\eps_0}(e^{\eps_0}-1)/n$.
\label{thm:local-to-centra-swap}
\end{thm}
\begin{proof}
The algorithm $\Asw$ defines a joint distribution between $I$ and the corresponding output sequence of~$\Asw$, which we denote by $Z_1,Z_2,\ldots,Z_n$. We first observe that $Z_{1:n}$ can be seen as the output of a sequence of $n$ algorithms with conditionally independent randomness: $\Balg^{(i)}\colon \cS^{(1)}\times \cdots \times \cS^{(i-1)} \times  \cD^n \to\cS^{(i)}$ for $i \in [n]$. On input $s_{1:i-1}$ and $D$, $\Balg^{(i)}$ produces a random sample from the distribution of $Z_i$ conditioned on $Z_{1:i-1} = s_{1:i-1}$. The outputs of $\Balg^{(1)},\ldots,\Balg^{(i-1)}$ are given as the input to $\Balg^{(i)}$. By definition, this ensures that random bits used by $\Balg^{(i)}$ are independent of those used by $\Balg^{(1)},\ldots,\Balg^{(i-1)}$ conditioned on the previous outputs. Therefore, in order to upper bound the privacy parameters of $\Asw$, we can analyze the privacy parameters of $\Balg^{(1)},\ldots,\Balg^{(n)}$ and apply the advanced composition theorem for differential privacy (Theorem \ref{thm:composition}).

Next we observe that, by the definition of $\Asw$, conditioned on $I$, $Z_i$ is the output of $\Aldp^{(i)}(s_{1:i-1}; x)$ with its internal randomness independent of $Z_{1:i-1}$ (where $x$ is the $i$-th element in $\sigma_I(D)$). In other words, for $i \geq 2$, $\Balg^{(i)}$ can equivalently be implemented as follows. First, sample an index $T$ from the distribution of $I$ conditioned on $Z_{1:i-1} = s_{1:i-1}$. Then, if $T=i$ output $\Aldp^{(i)}(s_{1:i-1}; x_1)$. Otherwise, output $\Aldp^{(i)}(s_{1:i-1}, x_i)$. To implement $\Balg^{(1)}$ we sample $T$ uniformly from $[n]$ and then output $\Aldp^{(1)}(x_T)$.

We now prove that for every $i\in [n]$, $\Balg^{(i)}$ is $\left(2e^{\eps_0}(e^{\eps_0}-1)/n,0\right)$-differentially private at index $1$. Let $D=x_{1:n}$ and $D'=(x'_1,x_{2:n})$ be two datasets that differ in the first element. Let $s_{1:i-1}$ denote the input to $\Balg^{(i)}$. We denote by $\mu$ the probability distribution of $\Balg^{(i)}(s_{1:i-1},D)$, denote by $\mu_0$ (or $\mu_1$) the probability distribution of $\Balg^{(i)}(s_{1:i-1},D)$ conditioned on $T\neq i$ ($T=i$, respectively) and by $p$ the probability that $T=i$ (where $T$ is sampled from the distribution of $I$ conditioned on $Z_{1:i-1} = s_{1:i-1}$ as described above). We also denote by $\mu',\mu'_0,\mu'_1$ and $p'$ the corresponding quantities when $\Balg^{(i)}$ is run on $D'$. By the definition, $\mu = (1-p)\mu_0 + p\mu_1$ and $\mu' = (1-p')\mu'_0 + p'\mu'_1$.

For $i=1$, $T$ is uniform over $[n]$ and hence $p=p'=1/n$. Further, $\mu_0$ is equal to $\mu'_0$ (since both are equal to the output distribution of $\Aldp^{(1)}(x_T)$ conditioned on $T\neq 1$). By $\eps_0$-local differential privacy of $\Aldp^{(1)}$ and quasi-convexity of $(\eps,\delta)$-DP we obtain that $\mu_0 \edclose{\eps_0}{0} \mu_1$. Therefore, privacy amplification by subsampling (Lemma~\ref{lem:pas}) implies that $\mu_0 \edclose{(e^\eps_0-1)/n}{0} \mu$. Similarly, we obtain that $\mu'_0 \edclose{(e^\eps_0-1)/n}{0} \mu'$ and therefore, by the triangle inequality,  $\mu \edclose{2(e^\eps_0-1)/n}{0} \mu'$. In other words, $\Balg^{(1)}$ is $(2(e^{\eps_0}-1)/n,0)$-differentially private at index $1$.

For $i \geq 2$, we again observe that $\mu_0 = \mu'_0$ since in both cases the output is generated by $\Aldp^{(i)}(s_{1:i-1}; x_i)$. Similarly, $\eps_0$-local differential privacy of $\Aldp^{(i)}$ implies that $\mu_0 \edclose{\eps_0}{0} \mu_1$ and
 $\mu'_0 \edclose{\eps_0}{0} \mu'_1$.

We now claim that $p \leq e^{2\eps_0}/n$. To see this, we first observe that the condition $Z_{1:i-1} = s_{1:i-1}$ is an event defined over the output space of $\Aloc$. Conditioning on $T = i$ reduces $\Asw$ to running $\Aloc$ on $\sigma_i(D)$. Note that for $j\neq i$, $\sigma_i(D)$ differs from $\sigma_j(D)$ in at most two positions. Therefore, by $\eps_0$-differential privacy of $\Aloc$ and group privacy (e.g.~\cite{dwork2014algorithmic}), we obtain that
\[
\frac{\Pr\left[Z_{1:i-1} = s_{1:i-1} \mid T = i  \right]}{\Pr\left[Z_{1:i-1} = s_{1:i-1}\mid T = j \right]} \leq e^{2\eps_0}.
\]
By quasi-convexity of $(\eps,\delta)$-DP we obtain that
\[
\frac{\Pr\left[Z_{1:i-1} = s_{1:i-1} \mid T = i  \right]}{\Pr\left[Z_{1:i-1} = s_{1:i-1}\right]} \leq e^{2\eps_0}.
\]
This immediately implies our claim since
$$	\Pr[T = i \mid Z_{1:i-1} = s_{1:i-1}] = \frac{\Pr\left[Z_{1:i-1} = s_{1:i-1} \mid T = i  \right] \cdot \Pr[T = i]}{\Pr\left[Z_{1:i-1} = s_{1:i-1}\right]} \leq \frac{1}{n} \cdot e^{2\eps_0}.$$

Privacy amplification by sampling implies that $\mu_0 \edclose{e^{2\eps_0}(e^\eps_0-1)/n}{0} \mu$. Applying the same argument to $D'$ and using the triangle inequality we get that $\mu \edclose{2e^{2\eps_0}(e^\eps_0-1)/n}{0} \mu'$.

Applying the advanced composition theorem for differential privacy with $\eps_1 = 2e^{2\eps_0}(e^{\eps_0}-1)/n$ and $n$ steps we get that $\Asw$ satisfies $(\eps,\delta)$-DP at index 1 for
\[
\eps \leq \eps_1\sqrt{2 n \log(1/\delta)} + n \eps_1 (e^{\eps_1}-1).
\]
\end{proof}

Finally, we describe a reduction of the analysis of $\Asl$ to the analysis of $\Asw$.
\begin{proof}[Proof of Theorem~\ref{thm:local-to-central}]
Let $D$ and $D'$ be two datasets of length $n$ that differ at some index $i^\ast \in [n]$. The algorithm $\Asl$ can be seen as follows. We first pick a random one-to-one mapping $\pi^\ast$ from $\{2,\ldots,n\} \to [n]\setminus \{i^\ast\}$ and let
\[
\pi^\ast(D) = (x_{i^\ast},x_{\pi^\ast(2)},\ldots,x_{\pi^\ast(n)}).
\]
Namely, we move $x_{i^\ast}$ to the first place and apply a random permutation to the remaining elements. In the second step we apply $\Asw$ to $\pi^\ast(D)$. It is easy to see that for a randomly and uniformly chosen $\pi^\ast$ and uniformly chosen $I\in [n]$ the distribution of $\sigma_I(\pi^\ast(D))$ is exactly a random and uniform permutation of elements in $D$.

For a fixed mapping $\pi^\ast$, the datasets $\pi^\ast(D)$ and $\pi^\ast(D')$ differ only in the element with index $1$. Therefore $\Asw(\pi^\ast(D)) \edclose{\eps}{\delta} \Asw(\pi^\ast(D'))$ for $\eps$ and $\delta$ given in Theorem \ref{thm:local-to-centra-swap}. Using the quasi-convexity of $(\eps,\delta)$-DP over a random choice of $\pi^\ast$ we obtain that
 $\Asl(D) \edclose{\eps}{\delta} \Asl(D')$.

To obtain the bounds in the special cases we note that for $\eps_0 \leq \ln(n/4)/3$ we have that $\eps_1 \leq 2e^{3\eps_0}/n \leq 1/2$ and thus $(e^{\eps_1}-1) \leq 3\eps_1/2$. Therefore  
\[
\eps \leq e^{2\eps_0}(e^{\eps_0}-1)\sqrt{\frac{8\log(1/\delta)}{n}} + 3n\eps_1^2/2 \leq e^{2\eps_0}(e^{\eps_0}-1)\sqrt{\frac{8\log(1/\delta)}{n}} + 6 e^{4\eps_0}(e^{\eps_0}-1)^2/n.
\]

Further, for $\eps_0 \leq 1/2$ we get that $\eps_1 \leq 8\eps_0/n$ and the first term
\[
\eps_1\sqrt{2 n \log(1/\delta)} \leq 8 \eps_0\sqrt{\frac{2 \log(1/\delta)}{n}}.
\]
Using the fact that $n \geq 1000$ we have that $\eps_1 \leq 1/250$. This implies that $e^{\eps_1}-1 \leq \frac{65}{64} \eps_1$ and using $n \geq 1000$ and $\delta \leq 1/100$, we get the following bound on the second term
\[
n \eps_1 (e^{\eps_1}-1) \leq \frac{65}{64}n \eps_1^2 \leq  \frac{65 \eps_0^2}n \leq \frac{2}{3} \eps_0 \sqrt{\frac{\log(1/\delta)}{n}}.
\]
Combining these terms gives the claimed bound.

\end{proof}

\begin{remark}
Our bound can also be easily stated in terms of Renyi differential privacy (RDP) \cite{mironov2017renyi}. Using the fact that $\eps$-DP implies $(\alpha,\alpha \eps^2/2)$-RDP  \cite{bun2016concentrated} and the composition properties of RDP \cite{mironov2017renyi}, we obtain that for all $\alpha \geq 1$, the output of $\Asl$ satisfies $(\alpha, 2 \alpha e^{4\eps_0}(e^{\eps_0}-1)^2/n)$-RDP. Such results can be used to obtain tighter $(\eps,\delta)$ guarantees for algorithms that produce multiple shuffled outputs.
\end{remark}

\subsection{Shuffling after local randomization}
\label{sec:shuffle_after}
The proof of Theorem~\ref{thm:local-to-central} relies crucially on shuffling the data elements before applying the local randomizers. However implementing such an algorithm in a distributed system requires trusting a remote shuffler with sensitive user data, thus negating the key advantage of the LDP model.
Conversely, even if shuffling is performed on a set of already-randomized LDP responses, no additional privacy guarantees will be achieved if some attribute of each report (e.g., the choice of a randomizer) can reveal the identity of the reporting user.

Fortunately, it is possible design software systems where reports are randomized before shuffling and in which the reports coming from large groups of users are indistinguishable, e.g., because they apply the same local randomizer.
In such constructions, the privacy of each user's report still have its privacy amplified, by a factor proportional to the square root of the cardinality of indistinguishable reports. This follows immediately from the fact that shuffling the responses from the same local randomizers is equivalent to first shuffling the data points and then applying the local randomizers.

We make this claim formal in the following corollary.
\begin{corollary}
For a domain $\cD$, let $\Aldp^{(i)}\colon \cS^{(1)}\times \cdots \times \cS^{(i-1)} \times  \cD\to\cS^{(i)}$ for $i \in [n]$ (where $\cS^{(i)}$ is the range space of $\Aldp^{(i)}$) be a sequence of algorithms such that $\Aldp^{(i)}$ is $\eps_0$-differentially private for all values of auxiliary inputs in $\cS^{(1)}\times \cdots \times \cS^{(i-1)}$.
Let $\Apost \colon\cD^n\to \cS^{(1)}\times \cdots \times \cS^{(n)}$ be the algorithm that given a dataset $D \in \cD^n$, computes $z_{1:n}=\Aloc(D)$, samples a random and uniform permutation and outputs $z_{\pi(1)},\ldots,z_{\pi(n)}$. Let $S\subseteq [n]$ be any set of indices such that for all $i,j\in S$, $\Aldp^{(i)}\equiv \Aldp^{(j)}$. Then for
$|S| \geq 1000$, $0<\eps_0<1/2$ and $0<\delta<1/100$ and every $i\in S$, $\Apost$ satisfies $(\eps,\delta)$-differential privacy at index $i$ in the central model, where $\eps = 12 \eps_0\sqrt{\frac{\log(1/\delta)}{|S|}}$.
\label{cor:local-to-central-post}
\end{corollary}
We note that for the conclusion of this corollary to hold it is not necessary to randomly permute all the $n$ randomized responses. It suffices to shuffle the elements of $S$. We also clarify that for $i < j$, by $\Aldp^{(i)}\equiv \Aldp^{(j)}$ we mean that for all sequences $z_{1:j-1}$ and $x\in \cD$, the output distributions of $\Aldp^{(i)}(z_{1:i-1},x)$ and $\Aldp^{(j)}(z_{1:j-1},x)$ are identical (and, in particular, the output distribution $\Aldp^{(j)}$ does not depend on~$z_{i:j-1}$).
\subsection{Lower Bound for Local DP Protocols}
\label{sec:lowerLDP}

The results of this section give us a natural and powerful way to prove lower bounds for protocols in the local differential privacy model. We can apply Theorem~\ref{thm:local-to-central} in the reverse direction to roughly state that for any given problem, lower bounds on the error of $\Omega(\alpha/\eps)$ (for some term $\alpha$ that might depend on the parameters of the system) of an $\eps$-centrally differentially private protocol translate to a $\Omega(\alpha\sqrt{n}/\eps)$ lower bound on the error of any $\eps$-locally differentially private protocol of the kind that our techniques apply to.

As an exercise, a lower bound of $\Omega(\sqrt{k} \,\polylog(d)/\eps)$ for the problem of collecting frequency statistics from users across time in the central DP framework with privacy guarantee $\eps$ directly implies that the result in Theorem~\ref{thm:ldpLongUtil} is tight. We note here that the results of Dwork et al.~\cite{Dwork-continual} do show a lower bound of $\Omega(\log(d)/\eps)$ for the setting when $k=1$ in the central DP framework. This strongly suggests that our bounds might be tight, but we cannot immediately use this lower bound as it is stated only for the pure $\eps$-differential privacy regime. It is an open problem to extend these results to the approximate differential privacy regime.

%% file: discussion.tex
\section{Discussion and Future Work} \label{sec:discussion}

Our amplification-by-shuffling result is encouraging, as it demonstrates
that the formal guarantees of differential privacy
can encompass intuitive privacy-enhancing techniques,
such as the addition of anonymity, 
which are typically part of existing, best-practice privacy processes.
By accounting for the uncertainty induced by anonymity,
in the central differential privacy model
the worst-case, per-user bound on privacy cost can be dramatically lowered.

Our result implies that
industrial adoption of LDP-based mechanisms may have 
offered much stronger privacy guarantees
than previously accounted for,
since anonymization of telemetry reports
is standard privacy practice in industry.
This is gratifying, since
the direct motivation for our work
was to better understand the guarantees
offered by one 
industrial privacy-protection mechanism: the
\emph{Encode, Shuffle, Analyze} (ESA) architecture
and \textsc{Prochlo} implementation of Bittau et al.~\citep{prochlo}.

However,
there still remain gaps
between our analysis and proposed practical, real-world mechanisms,
such as those of the ESA architecture.
In particular,
our formalization assumes the user population to be static,
which undoubtedly it is not.
On a related note,
our analysis
assumes that (most) all users send reports at each timestep
and ignores
the privacy implications of timing or traffic channels,
although
both must be considered,
since reports may be triggered by privacy-sensitive events on users' devices,
and it is infeasible to send all possible reports at each timestep.
The ESA architecture addresses traffic channels using
large-scale batching 
and randomized thresholding of reports, with elision,
but any benefits from that mitigation 
are not included in our analysis.

Finally,
even though it is a key aspect of the ESA architecture,
our analysis does not consider how users
may fragment their sensitive information
and at any timestep send multiple LDP reports, one for each fragment,
knowing that each will be anonymous and unlinkable.
The splitting of user data into carefully constructed fragments to increase users' privacy 
has been explored for specific applications,
e.g., by 
Fanti et al.~\cite{rapporunknowns} 
which fragmented users' string values into overlapping $n$-grams, 
to bound sensitivity while enabling an aggregator to reconstruct popular user strings.
Clearly, such fragmentation should be able to
offer significantly improved privacy/utility tradeoffs,
at least in the central model.
However, in both the local and central models of differential privacy,
the privacy implications 
of users' sending LDP reports 
about disjoint, overlapping, or equivalent fragments of their sensitive information
remain to be formally understood, in general.


